\documentclass[11pt]{article}

\usepackage{graphicx}
\usepackage{subfigure}
\usepackage{algpseudocode}
\usepackage{algorithm}
\usepackage{amsmath}
\usepackage{amssymb}
\usepackage{natbib}
\usepackage{appendix}
\usepackage[affil-it]{authblk}
\usepackage{amsopn}
\usepackage{mathtools}
\usepackage{float}
\usepackage{bbm}
\usepackage{subfigure}
\usepackage{appendix}
\usepackage{comment}
\usepackage{amsthm}

\theoremstyle{definition}
\newtheorem{lemma}{Lemma}
\newtheorem{theorem}{Theorem}
\newtheorem{corollary}{Corollary}

\newtheorem{definition}{Definition}
\numberwithin{theorem}{section} 
  
\DeclareMathOperator{\tr}{Tr}
\DeclareMathOperator{\argmin}{argmin}
\DeclareMathOperator{\Span}{Span}
\DeclareMathOperator{\dist}{dist}

\DeclareMathOperator{\Range}{\mathcal{R}}
\DeclareMathOperator{\Nullspace}{\mathcal{N}}
\DeclareMathOperator{\stack}{\mathbf{Vec}}

\makeatletter
\DeclareMathOperator{\r@ank}{\mathrm{rank}}
\newcommand{\rank}[1]{\r@ank\kern-1pt\left(#1\right)}
\makeatother

\renewcommand{\vec}[1]{\mathbf{#1}}
\newcommand{\widehatvec}[1]{\widehat{\mathbf{#1}}}
\newcommand{\vecwidehat}[1]{\vec{\widehat{#1}}}
\newcommand{\widetildevec}[1]{\widetilde{\vec{#1}}} 

\newcommand{\doublewidetildevec}[1]{\widetilde{\kern0pt\raisebox{0pt}[0.85\height]{$\widetilde{\vec{#1}}$}}}

\newcommand{\norm}[1]{\left\lVert#1\right\rVert}
\newcommand{\diag}{\mathrm{diag}}
\newcommand{\ones}{\mathbbm{1}}

\title{On the existence of the maximum likelihood estimate and convergence rate under gradient descent for multi-class logistic regression}

\author[1]{Dwight Nwaigwe}
\affil[1]{Program in Applied Mathematics, The University of Arizona, Tucson, USA}
\author[2]{Marek Rychlik}
\affil[2]{Department of Mathematics, The University of Arizona, Tucson, USA}

\begin{document}

\maketitle

\begin{abstract}
We revisit the problem of the existence of the maximum likelihood estimate for multi-class logistic regression. We show that one method of ensuring its existence is by assigning positive probability to every class in the sample dataset. The notion of data separability is not needed, which is in contrast to the classical set up of multi-class logistic regression in which each data sample belongs to one class. We also provide a general and constructive estimate of the convergence rate to the maximum likelihood estimate when gradient descent is used as the optimizer. Our estimate involves bounding the condition number of the Hessian of the maximum likelihood function. The approaches used in this article rely on a simple operator-theoretic framework.
\end{abstract}

\section{Introduction}

    Multi-class logistic regression is one of the most common statistical models. The existence of the maximum likelihood estimate (MLE) has been of long-standing interest.  Its existence is not of mere theoretical interest as there are several works which are dependent on the existence of a MLE, for instance elastic weight consolidation  (\cite{kirkpatrick}). \cite{silvapulle} was the first to report results on the existence of the MLE in the case of binary logistic regression and \cite{albert} extended this to the case of several possible outcomes, i.e. multi-class logistic regression. In \cite{albert}, a Euclidean geometric approach is taken, with the fundamental idea being that of data separability, or quasi-separability.

    Following \cite{albert}, we say that a dataset corresponding to $C$ classes is completely separable if for any sample $\vec{x}$ belonging to class $j$, there exists a matrix $\vec{W}$ with $C$ rows such that $\left( \vec{W} \vec{x} \right)_j- \left( \vec{W} \vec{x} \right)_t >0$, and where $1 \leq j, t \leq C$. A theorem from \cite{albert} states that a dataset which is separable corresponds to the nonexistence of the MLE. Another theorem states that if the dataset is \textit{quasicompletely separated}  (which we do not define here) then again the MLE does not exist. The third theorem states that an MLE exists in in the absence of these conditions. The Euclidean geometric approach of these three theorems tend to not provide an obvious answer to the existence of an MLE.  This difficulty was noted in \cite{albert} which mentioned the possibility of using a linear program. Linear programming methods were more thoroughly investigated in \cite{santner}, \cite{santner2}, \cite{clarkson}, and \cite{konis}.  A second approach to deal with the intransigence of the problem of determining data separability is to make probabilistic statements on the existence of the MLE, as done in \cite{candes} and \cite{candes2} for the case of a two-class problem.

    One may consider a generalization of the usual multi-class logistic regression by allowing the sample data to belong to all classes, albeit with varying probabilities. We call this label \textit{smoothing}. We then ask if the MLE exists. We address this question in this work, and answer affirmatively.  Moreover, in contrast to the previous works, we do not impose a requirement of data separability or the full rank of the data matrix. Given that an MLE exists, one typically seeks to find it by using a numerical optimization method. In the case of small datasets, optimizers with a quadratic convergence rate such as Newton-Raphson are typically used. When datasets are very large, as is often the case in many modern datasets, or in the machine learning community, optimizers which are linear in convergence rate are used, an example being gradient descent. This provides motivation for our study of the optimization of the MLE problem using gradient descent as the optimizer. Prior studies (\cite{freund2018condition}, \cite{nacson19a}, \cite{nacson19b},\cite{ji}) on the convergence of gradient descent for logistic regression assume data separability and binary classification. We note that according to the results in \cite{albert}, \cite{silvapulle}, data separability and binary classification imply that the MLE does not exist-therefore these cases are not relevant to our scenario. To address the convergence rate we investigate spectral properties of the Hessian of the MLE and as a consequence we provide the convergence rate in terms of a desired contraction rate.


\section{Notation and setup}

Throughout the paper we will consider matrices of various sizes.
The vector space of all $p$ by $q$ ($p\times q$) matrices will be denoted
by $L(\mathbb{R}^q,\mathbb{R}^p)$ and every such matrix $A$ is identified
with a linear operator $A:\mathbb{R}^p\to\mathbb{R}^q$.

Assume we have a matrix $\mathbf{X}\in L(\mathbb{R}^N,\mathbb{R}^D)$ whose columns
represent a sample, and a target matrix $\mathbf{T}\in L(\mathbb{R}^D,\mathbb{R}^C)$
where an entry $t_i^{(n)}$ ($i^{th}$ row and $n^{th}$ column
of $\mathbf{T}$ ) is the probability that the $n^{th}$ sample
($n^{th}$ column of $\mathbf{X}$) belongs to class $i$. We
require $\sum_{i=1}^{C} t^{(n)}_i=1, \ t^{(n)}_i \geq 0$. In
the simplest example of multi-class logistic regression, each
column of $\mathbf{X}$ belongs to one class so that
$t_i^{(n)} \in \{0,1\}$ for $\ 1\leq i \leq C$ where $C$ is
the number of classes. Define
$y^{(n)}_i=\sigma^{(i)} \left(\mathbf{W}\mathbf{x}^{(n)} \right) $, where
$\boldsymbol\sigma: \mathbb{R}^C\to \mathbb{R}^C$ is given by the
formula
\begin{equation*}
\sigma^{(i)}(\mathbf{u}) = \frac{e^{u_i}}{\sum_{j=1}^C e^{u_j}}.
\end{equation*}
and where $\mathbf{W}\in L(\mathbb{R}^D,\mathbb{R}^C)$. The function
$\boldsymbol\sigma$ is known as ``softmax". Consider the
function $L: L(\mathbb{R}^D,\mathbb{R}^C)\to\mathbb{R}$ given by
\begin{equation}
\label{eqn:loss-function}
L(\mathbf{W};\mathbf{X},\mathbf{T}) = -\sum_{n=1}^N \sum_{i=1}^C t_i^{(n)} \log y_i^{(n)}.
\end{equation} 
In this equation $\mathbf{W}$ plays a role of a variable, and
$\mathbf{X}$, $\mathbf{T}$ are merely parameters. When it is clear
from the context what $\mathbf{X}$ and $\mathbf{T}$ is, we will
simply write $L(\mathbf{W})$ instead of
$L(\mathbf{W};\mathbf{X},\mathbf{T})$, or $L(\mathbf{W};\mathbf{X})$ when it
is clear what $\mathbf{T}$ is.  On occasions, however, we will
need to consider $L$ with different values of $\mathbf{X}$ or
$\mathbf{T}$ in the same context, and then we will use the
notation $L(\mathbf{W};\mathbf{X})$ or even
$L(\mathbf{W};\mathbf{X},\mathbf{T})$.  The quantity
$L(\mathbf{W};\mathbf{X},\mathbf{T})$ defined by
equation~\eqref{eqn:loss-function} is the \emph{negative log-likelihood} and is commonly known as the
\emph{cross-entropy} in the machine learning community. The problem of minimizing this function (or equivalently in the neural network community- training the neural network) is the
problem of finding the optimal weight matrix $\widehatvec{W}$
which minimizes $L(\mathbf{W})$:
\begin{equation}
\label{eq:argmin}
\widehatvec{W} = \argmin_{\mathbf{W}\in L(\mathbb{R}^D,\mathbb{R}^C)} L(\mathbf{W}).
\end{equation}
 In other words, \eqref{eq:argmin} is
    equivalent to finding $\mathbf{W}$ such that the probability of
    observing the samples is maximized, where $y_i^{(n)}$ is the
    computed probability that the $n^{th}$ sample belongs to class
    $i$.

In \cite{patternpaper}, the formulas for the Fr\'echet
    derivative, gradient, and second Fr\'echet derivative of
    $L(\mathbf{W})$ are given respectively as:
\begin{align}
\label{eqn:frechet-formulas}
& DL(\mathbf{W})\mathbf{V} =\sum_{n=1}^{N}  
-(\mathbf{t}^{(n)}-\mathbf{y}^{(n)})^\intercal\,\mathbf{V}\,\mathbf{x}^{(n)}, \\
\label{eqn:gradient-formula}
& \nabla L(\mathbf{W}) = 
-\sum_{n=1}^N\left(\mathbf{t}^{(n)}-\mathbf{y}^{(n)}\right)\left(\mathbf{x}^{(n)}\right)^\intercal
= -(\mathbf{T}-\mathbf{Y})\,\mathbf{X}^\intercal, \\
\label{eq:second-derivative}
& D^2L(\mathbf{W})(\mathbf{U},\mathbf{V}) = \sum_{n=1}^{N}  
{\mathbf{x}^{(n)}}^\intercal\,\mathbf{U}^\intercal \mathbf{Q}^{(n)} 
\mathbf{V}\,\mathbf{x}^{(n)}.    
\end{align}
where $\mathbf{Q}^{(n)}=\diag(\mathbf{y}^{(n)}) - 
\mathbf{y}^{(n)}\,{\mathbf{y}^{(n)}}^\intercal$.
We need the Fr\'echet derivative, gradient, and second Fr\'echet derivative to study $L$ with respect to a space $Z$ that is defined shortly. We also recall that a number of fundamental properties of $L$ and its derivatives
    were shown.  We summarize them below, along with additional background-type facts.
    \begin{enumerate}
    \item The quadratic form induced by the bilinear form
      $D^2L(\mathbf{W})$ given by \eqref{eq:second-derivative} is
      non-negative definite for all $\mathbf{W}$.
    \item If $\mathbf{X}$ has rank $D$ with $D \leq N$ (the condition
      $D \leq N$ means the number of samples is large compared to
      the dimension of each sample) then the quadratic form
      induced by $D^2L(\mathbf{W})$ is positive definite on the
      subspace $Z\subseteq L(\mathbb{R}^D,\mathbb{R}^C)$ with column means
      equal to 0. This subspace can also be defined arithmetically
      by
      \[ Z := \{\mathbf{W}\in L(\mathbb{R}^D,\mathbb{R}^C) \,:\, \ones^\intercal \mathbf{W} = 0\} \]
      where $\ones\in\mathbb{R}^C$ is the vector of 1's. More
      precisely, there is a constant $ b >0$ such that for every
      $\mathbf{W}\in L(\mathbb{R}^D,\mathbb{R}^C)$ and $\mathbf{V}\in Z$
      \[ D^2L(\mathbf{W})(\mathbf{V},\mathbf{V}) \ge b \|\mathbf{V}\|^2. \]
      The choice of the norm is immaterial,
      but some calculations are facilitated by using the Frobenius
      norm.
    \item The function $L$ possesses a shift invariance
      property. More precisely, let us consider an orthogonal
      decomposition
      \[ L(\mathbb{R}^D,\mathbb{R}^C) = Z \oplus \Gamma \]
      where
      $\Gamma = Z^\perp$ is the orthogonal complement. This vector
      subspace can be given more explicitly as
      \[ \Gamma = \{ \ones\cdot \mathbf{c}^\intercal\,:\, \mathbf{c}\in\mathbb{R}^C \}. \]
      These are exactly the matrices which have identical entries
      in each column. The shift invariance is expressed as
      follows: for every $\mathbf{W}\in L(\mathbb{R}^D,\mathbb{R}^C)$ and
      every $\mathbf{c}\in\mathbb{R}^D$
      \[ L\left(\mathbf{W} + \ones\cdot\mathbf{c}^\intercal\right) = L(\mathbf{W}). \]
      Therefore it is sufficient to study $L$ restricted to $Z$,
      which will be denoted by $L|Z$.
    \item The fact that $D^2L(\mathbf{W})$ induces a positive
      definite quadratic form on $Z$ implies that $L|Z$ is a
      \emph{locally strongly convex function} (cf. \ref{sec:technical-lemmas}). It is not true that
      a locally strongly convex function must have a global minimum
      (an explicit example is given in
      \cite{patternpaper}). However, a necessary and sufficient
      condition for $L|Z$ (and thus $L$) to have a global minimum
      is that there exists a critical point of $L$: a $\mathbf{W}$
      such that $\nabla L(\mathbf{W})=0$.  Furthermore, if a critical
      point exists then
      \[ \lim_{\mathbf{V}\to\infty,\mathbf{V}\in Z} L(\mathbf{V}) = \infty \].
      It should be noted that this condition is necessary and
      sufficient for a convex function $L$ to have a unique global
      minimum on a subspace $Z$.
    \end{enumerate}

    To reiterate, in the current paper we address two important issues:
    \begin{enumerate}
    \item The existence of the MLE (i.e., a global minimum of 
      $L$) (section~\ref{sec:existence-of-minimum}).
    \item Effective bounds on the convergence of algorithms
      which find the MLE. For instance, 
      one may then use the gradient descent formula to minimize $L$:
      \begin{equation}
        \label{eq:grad-descent}
        \mathbf{W}_{n} = \mathbf{W}_{n-1} - \eta\,\boldsymbol\nabla L(\mathbf{W}_{n-1}).
      \end{equation}
      The speed of convergence is given in terms of the condition
      number of the Hessian matrix of $L$ at the minimum. Equivalently
      we may seek constants $b, B\in\mathbb{R}$, $0<b \le M <\infty$, such that
      for every $\mathbf{V}\in Z$ we have:
      \[ b\,\|\mathbf{V}\|^2\le D^2L(\mathbf{W})(\mathbf{V},\mathbf{V}) \le B\, \|\mathbf{V}\|^2. \]
      It will be seen that such bounds exist and can be constructively found
      (section~\ref{sec:rate-of-convergence}).
    \end{enumerate}

\section{Existence of the MLE}
    \label{sec:existence-of-minimum}

    The immediate objective is to prove that $L$
    has a critical point under the condition where each sample has non-zero probability for all classes.
    For the sake
    of clarity, we adopt some definitions and notations:
    \begin{definition}[Positivity of a Matrix]
      A positive matrix is a real matrix $\mathbf{A} = [a_{ij}]$ which
      is positive element-wise: $a_{ij}>0$ for all $i$ and $j$.
      We write $\mathbf{A}>0$ iff $A$ is a positive matrix.
    \end{definition}
    \begin{definition}[Nullspace and Range of a Linear Operator]
      For a linear operator $\mathbf{F}$, let $\Nullspace(\mathbf{F})$ and $\Range(\mathbf{F})$ denote
      the nullspace (kernel) and range (image) of $\mathbf{F}$, respectively.
    \end{definition}
    Clearly, a sufficient condition for $\nabla L(\mathbf{W})=0$ is
    that $\mathbf{T}=\mathbf{Y}$ where
    $\mathbf{Y}=\boldsymbol\sigma(\mathbf{W}\mathbf{X})$.  Also, $\mathbf{T}>0$ follows from
    $\mathbf{T}=\mathbf{Y}$. However, it is possible to have a minimum
    $\mathbf{W}$ for which $\mathbf{T}\neq\mathbf{Y}$. It is also possible
    to have a global minimum for $\mathbf{T}\not>0$.  However, as
    $\nabla L(\mathbf{W})=-(\mathbf{T}-\mathbf{Y})\,\mathbf{X}^\intercal=0$, the
    necessary and sufficient condition for $\mathbf{W}$ to be a
    critical point is:
    \begin{equation}
      \label{eqn:minimality-criterion}
      \Nullspace( \mathbf{T}-\mathbf{Y} ) \supseteq  \Range\left(\mathbf{X}^\intercal\right) = \Nullspace(\mathbf{X})^\perp.
    \end{equation}
    The following lemma fully resolves the issue of the existence
    and calculating the minimum in the easiest, but still useful,
    case:
\begin{lemma}\label{theorem:minimum1}	
      Assume $\mathbf{T}>0$, $N=D$, and that $\mathbf{X}$ is invertible. 
      Then a minimum of $L$ exists and every minimum $\widetildevec{W}$ is given by
      \[ \widetildevec{W}= \mathbf{R}\mathbf{X}^{-1} + \ones\mathbf{c}^\intercal\] 
      where $\mathbf{R} = \ln(\mathbf{T})$ (elementwise logarithm) and $\mathbf{c}\in\mathbb{R}^D$ is arbitrary.
      Exactly one of the minima belongs to $Z$ and is
      \[ \widetildevec{W} = \left(\mathbf{I} - \frac{1}{C}\ones\ones^\intercal\right)\mathbf{R}\mathbf{X}^{-1}\]
      and is also the matrix obtained from $\mathbf{R}\mathbf{X}^{-1}$ obtained by subtracting from each
      column its mean.
\end{lemma}
\begin{proof}
      Suppose that $\nabla L(\mathbf{W}) = 0$.  As in this case the
      operator $\mathbf{X}^\intercal$ is invertible as well, and
      therefore surjective:
      $\Range\left(\mathbf{X}^\intercal\right)= \mathbb{R}^N$ (the entire
      codomain). Hence, by Lemma~\ref{eqn:minimality-criterion},
      $\Nullspace(\mathbf{T}-\boldsymbol\sigma(\mathbf{W}\mathbf{X})) = \mathbb{R}^N$,
      which implies $\mathbf{T}-\boldsymbol\sigma(\mathbf{W}\mathbf{X})=0$.
      Thus, $\mathbf{T} = \boldsymbol\sigma(\mathbf{W}\mathbf{X})$. Also, $\mathbf{T}=\boldsymbol\sigma(\mathbf{R})$ 
      and therefore $\boldsymbol\sigma(\mathbf{R}) = \boldsymbol\sigma(\mathbf{W}\mathbf{X})$.
      Lemma~\ref{lemma:sigma-translational-invariance}(appendix) yields :
      \begin{equation}
       \label{eq:minimum1}
        \mathbf{R}=\mathbf{W}\mathbf{X} + \ones\cdot\mathbf{c}^\intercal
      \end{equation}
      for some $\mathbf{c}\in\mathbb{R}^D$. Therefore,
      \begin{equation}
        \label{eq:minimum11}
        \mathbf{W} = \mathbf{R}\mathbf{X}^{-1} - \ones\cdot\mathbf{c}^\intercal\mathbf{X}^{-1}
        = \mathbf{R}\mathbf{X}^{-1} - \ones\cdot\mathbf{d}^\intercal
      \end{equation}
      where $\mathbf{d} = (\mathbf{X}^{-1})^\intercal\mathbf{c}$ is also
      arbitrary. Clearly, the only way to put $\mathbf{W}$ in $Z$ is
      to pick $\mathbf{d}$ to be the vector of column means of
      $\mathbf{R}\mathbf{X}^{-1}$. Formally, assuming $\mathbf{W}\in Z$ and
      multiplying~\eqref{eq:minimum11} by $\ones^\intercal$ on the
      left we obtain:
      \[ 0 = \ones^\intercal\mathbf{W} = \ones^\intercal\mathbf{R}\mathbf{X}^{-1} - \ones^\intercal\ones\mathbf{d}^\intercal
        =\ones^\intercal\mathbf{R}\mathbf{X}^{-1} - C\mathbf{d}^\intercal.\]
      (Note: $C$ is the number of classes.) 
      Hence
      \[ \mathbf{d}^\intercal = \frac{1}{C}\ones^\intercal\mathbf{R}\mathbf{X}^{-1} \]
      i.e. the row vector of column means of $\mathbf{R}\mathbf{X}^{-1}$. Plugging into equation~\eqref{eq:minimum11} we obtain
      \[ \mathbf{W} = \mathbf{R}\mathbf{X}^{-1} - \frac{1}{C}\ones\ones^\intercal\mathbf{R}\mathbf{X}^{-1} =
        \left(\mathbf{I} - \frac{1}{C}\ones\ones^\intercal\right)\mathbf{R}\mathbf{X}^{-1}          \]
      as claimed. The operator $\mathbf{I} - (1/C)\ones\ones^\intercal$ is recognized as the orthogonal projection
      on the space of vectors of mean $0$.
    \end{proof}

\begin{corollary}\label{corollary:minimum2}	
      If $\mathbf{T} > 0$ and $\rank{\mathbf{X}}=D$, then a global
      minimum of $L = L(\cdot; \mathbf{X},\mathbf{T})$ exists. Furthermore, the global minimum
      exists and is unique within subspace $Z$.
\end{corollary}
\begin{proof} 
      Consider an invertible submatrix
      $\widetilde{\mathbf{X}}\in L(\mathbb{R}^D,\mathbb{R}^D)$ of $\mathbf{X}$ along
      with the corresponding function
      $\widetilde{L}=L(\cdot;\widetilde{\mathbf{X}})$. Let
      $\widetilde{\widetilde{\mathbf{X}}}$ be the complementary matrix of
      $\widetilde{\mathbf{X}}$ within $\mathbf{X}$ and let
      $\widetilde{\widetilde{L}}=L(\cdot;\widetilde{\widetilde{\mathbf{X}}})$. 
      Then
      $L = \widetilde{L}+\widetilde{\widetilde{L}}$.  We can use
      Lemma~\ref{theorem:minimum1} to deduce that $\tilde{L}$ has
      a unique global minimum on $Z$ and is also locally strongly
      convex on $Z$.  By
      Lemma~\ref{lemma:criterion-of-unique-global-minimum}
      \[ \lim_{\mathbf{W}\to\infty,\mathbf{W}\in Z} \widetilde{L}(\mathbf{W}) = 
      \infty. \]
      The function $\widetilde{\widetilde{L}}$ is
      bounded below by 0 and also convex (but perhaps not locally strongly
      convex). Therefore,
      \[ \lim_{\mathbf{W}\to\infty,\mathbf{W}\in Z} L(\mathbf{W})=\infty. \]
      Applying
      Lemma~\ref{lemma:criterion-of-unique-global-minimum} again
      we deduce that $L(\mathbf{W})$ has a unique global
      minimum within $Z$.  Therefore $L$ has global minima
      differing by a matrix of the form $\ones\cdot\mathbf{c}$.
    \end{proof}
The following theorem broadens the previous statements on the existence of the minimum: 

\begin{theorem}[Existence of minimum, $\rank{\mathbf{X}}=D$]
          \label{thm:existence-of-minimum-rank-D}
          Let us assume that $\mathbf{T} > 0$ and that
          $\rank{\mathbf{X}}=D$. Then a minimum of
          $L = L(\cdot;\mathbf{X},\mathbf{C})$ exists. Furthermore a unique
          minimum exists within a subspace $Z$. All minima of $L$
          may be obtained by additionally translating by a matrix
          of the form $\ones\cdot\mathbf{c}^\intercal$,
          $\mathbf{c}\in\mathbb{R}^D$.
\end{theorem}
\begin{proof}
      The idea is to study the behavior of $L(\beta\mathbf{W})$ as $\beta\;\to\;\infty$.
      According to Lemma~\ref{lemma:asymptotic-behavior-sigma}, if $\mathbf{u}^{(n)} = \mathbf{W}\mathbf{x}^{(n)}$, we have:
      \begin{align*}
        \lim_{\beta\to\infty} L(\beta\mathbf{W})
        &=-\sum_{n=1}^N \sum_{i=1}^C t_i^{(n)} \lim_{\beta\to\infty}\log \sigma^{(i)}\left(\mathbf{u}^{(n)}\right)\\
        &=\sum_{n=1}^N \sum_{i=1}^C t_i^{(n)} \lim_{\beta\to\infty}\left(-\log\sigma^{(i)}\left(\mathbf{u}^{(n)}\right)\right).
      \end{align*}
      As all terms in the sum are non-negative, the limit is infinite iff
      there is an $i$ and $n$ such that
      \begin{enumerate}
      \item $t_i^{(n)} > 0$, which is automatically guaranteed if $\mathbf{T}>0$, and
      \item 
      \begin{align*}
         &\lim_{\beta\to\infty}\left(-\log\sigma^{(i)}\left(\mathbf{u}^{(n)}\right)\right) = \infty,\\
         & \text{or, equivalently,} \\
         &\lim_{\beta\to\infty}\sigma^{(i)}\left(\mathbf{u}^{(n)}\right) = 0.
      \end{align*}  
      \end{enumerate}
       According to Lemma~\ref{lemma:asymptotic-behavior-sigma}, the
      second condition is satisfied when $i\notin J_n$, where
      $J_n$ is the set of indices $i$ for which $u_i^{(n)}$ is
      maximal. Thus, there must be a row $\mathbf{W}_i$ of $\mathbf{W}$
      such that $\mathbf{W}_i\mathbf{x}^{(n)}$ is not maximal. Only when
      all vectors $\mathbf{W}\mathbf{x}^{(n)}$ are multiples of $\ones$
      this is not possible. Thus, any matrix $\mathbf{W}$ such that
      $L(\beta\mathbf{W})$ is bounded as $\beta\to\infty$ satisfies
      \[ \mathbf{W}\mathbf{X} = \ones\mathbf{c}^\intercal\]
      for some $\mathbf{c}\in\mathbb{R}^N$. If we additionally assume
      that $\mathbf{W}\in Z$, we will show that $\mathbf{W}\mathbf{X} = 0$.
      The argument consists in pre-multiplying by $\ones^\intercal$:
      \[ 0=\ones^\intercal\mathbf{W}\mathbf{X} =
        \ones^\intercal\ones\mathbf{c}^\intercal = C\mathbf{c}^\intercal.\]
      (Note: $C$ is the number of classes.) 
      Hence $\mathbf{c}=0$ and $\mathbf{W}\mathbf{X} = 0$ as claimed.

      The condition $\mathbf{W}\mathbf{X} = 0$ may be rephrased as
      $\Nullspace(\mathbf{W}) \supseteq \Range(\mathbf{X})$. As $\dim \Range(\mathbf{X}) = \rank{\mathbf{X}}=D$,
      we have $\Range(\mathbf{X})=\mathbb{R}^D$, i.e. $\mathbf{X}$ is a surjective as a linear operator.
      There $\Nullspace(\mathbf{W})=\mathbb{R}^D$, i.e. $\mathbf{W}=0$. We thus have shown that for every $\mathbf{W}\in Z \setminus\{0\}$,
      $\lim_{\beta\to\infty}L(\beta\mathbf{W}) = \infty$. By 
      Lemma~\ref{lemma:criterion-of-unique-global-minimum} there is
      a unique global minimum of $L|Z$.
    \end{proof}
    When $\rank{\mathbf{X}}$ is arbitrary, a similar result holds,
    but it requires the use of some abstract linear algebra, both
    in its formulation as well as in the proof. Thus, we separated
    it from the more simple-minded
    Theorem~\ref{thm:existence-of-minimum-rank-D}.

    Occasionally we use the notion of orthogonality in the space of matrices $L(\mathbb{R}^p,\mathbb{R}^q)$.
    This requires us to define a scalar product. The only product used in the current paper is the Frobenius inner
    product:
    \[ \langle \mathbf{U}, \mathbf{V} \rangle = \tr{\mathbf{U}^T\mathbf{V}} = \sum_{i=1}^q\sum_{j=1}^p u_{ij}v_{ij}. \]

\begin{theorem}[Existence of minimum, $\rank{\mathbf{X}}$ arbitrary]
      \label{thm:existence-of-minimum-arbitrary-rank}
      Let us assume that $\mathbf{T} > 0$ and $\mathbf{X}\in L(\mathbb{R}^N,\mathbb{R}^D)$ be an arbitrary matrix.
      Then a global minimum of $L=L(\cdot;\mathbf{X},\mathbf{T})$ exists. Furthermore,
      \begin{enumerate}
      \item let $\mathbf{F}_\mathbf{X}\,:\,L(\mathbb{R}^D,\mathbb{R}^C)\to L(\mathbb{R}^N,\mathbb{R}^C)$ be a linear operator given by \newline
        $\mathbf{F}_\mathbf{X}(\mathbf{W})=\mathbf{W}\mathbf{X}$;
      \item let $Z_0\subseteq Z$ be a subspace of $Z$ defined by $Z_0 = 
      \Nullspace\left(\mathbf{F}_{\mathbf{X}}\right)\cap Z$;
      \item let $Z_1\subseteq Z$ be the complement of $Z_0$, so that $Z=Z_0\oplus Z_1$ (direct sum).
      \end{enumerate}
      Then
      \begin{enumerate}
      \item $L|Z $ is invariant under shift by a vector in $Z_0$;
      \item $L|Z_1$ has a unique global minimum;
      \item all global minima of $L|Z$ can be obtained by translating 
        the minimum of $L|Z_1$ by vectors in $Z_0$;
      \item all minima of $L$ may be obtained by adding  a matrix of
        the form $\ones\cdot\mathbf{c}^\intercal$, $\mathbf{c}\in\mathbb{R}^D$ to
        a minimum of $L|Z$.
      \end{enumerate}
    \end{theorem}
    \begin{proof}
      The proof remains identical to the proof of
      Theorem~\ref{thm:existence-of-minimum-rank-D}, until the
      final stage, when we can no longer claim that
      $\mathbf{W}\mathbf{X}=0$ implies $\mathbf{W}=0$.  Thus, we modify the
      proof from this point on.

      Let $\Gamma_1 = \{\ones\cdot\mathbf{c}^\intercal\,:\,\mathbf{c}\in\mathbb{R}^N\}$ be a
      subspace of $L(\mathbb{R}^N,\mathbb{R}^C)$. We have defined a
      similar space $\Gamma \subseteq L(\mathbb{R}^D,\mathbb{R}^C)$ which differs only by the dimensions
      of the matrices, which is the orthogonal complement of $Z$.  It is easy to see that
      $\mathbf{F}_\mathbf{X}(\Gamma) \subseteq \Gamma_1$ because $\mathbf{F}_\mathbf{X}(\ones\mathbf{c}^\intercal) = \ones\mathbf{c}^\intercal\mathbf{X}=
      \ones(\mathbf{X}^\intercal\mathbf{c})^\intercal\in\Gamma_1$.  Let
      $\Gamma_2 = \mathbf{F}_{\mathbf{X}}^{-1}(\Gamma_1)$
      ($\Gamma_2 \supseteq\Gamma$) be a vector subspace of
      $L(\mathbb{R}^D,\mathbb{R}^C)$.  We claim that 
      $\Gamma_2\cap Z \subseteq \Nullspace(\mathbf{F}_{\mathbf{X}})$. Indeed, we
      have shown that $\mathbf{F}_{\mathbf{X}}(\mathbf{W}) = \ones\cdot\mathbf{c}^\intercal$ and
      $\mathbf{W}\in Z$ implies that $\mathbf{W}\in \Nullspace\left(\mathbf{F}_{\mathbf{X}}\right)$.  The
      consequence is that $L|Z$ can be factored through the
      natural projection onto the quotient space
      $Z/(\Gamma_1\cap Z)$, which is the ``right'' domain of $L$.
      In fact, it is the same trick that resulted in introduction
      of $Z$: $L$ factored through the natural projection onto
      $L(\mathbb{R}^D,\mathbb{R}^C)/\Gamma$ (another way to understand
      shift invariance with respect to shifts by elements of
      $\Gamma$).  But also $L|Z$ is invariant under shifts by
      elements of $\Gamma_1\cap Z$.  Indeed, since
      $\mathbf{Y}=\left(\boldsymbol\sigma\circ\mathbf{F}_{\mathbf{X}}\right)(\mathbf{W})$
      then $L$ depends on $\mathbf{W}$ only through $\mathbf{Y}$. Since
      $\mathbf{F}_{\mathbf{X}}$ is a linear operator, $\mathbf{Y}$, and
      therefore $L$, are invariant under shifts by vectors in
      $\Nullspace(\mathbf{F}_{\mathbf{X}})$. And $L|Z$ is invariant under shits
      by vectors in $\Nullspace(\mathbf{F}_{\mathbf{X}})\cap Z$.

      It can be seen that $Z_0=\Nullspace(\mathbf{F}_{\mathbf{X}})\cap Z$ represents
      ``wasted parameters'' of the model, as there is no reduction
      of $L$ by descending along the directions belonging to this
      subspace (in fact, $L$ is constant along those
      directions). Eliminating these parameters leads to
      considering $L$ on the subspace $Z_1\subseteq Z$.  The
      function $L|Z_1$ does not contain any directions from
      $\Gamma_1$ and thus
      \[ \lim_{\mathbf{W}\to\infty,\mathbf{W}\in Z_1} L(\mathbf{W}) = \infty. \]
      Applying
      Lemma~\ref{lemma:criterion-of-unique-global-minimum} we
      deduce that $L$ has a global minimum in $Z_1$. This minimum
      could be non-unique because \emph{a priori} we do not know
      that $L|Z_1$ is locally strongly convex without assuming that
      $\rank{\mathbf{X}}=D$. However locally strong convexity still holds, and 
      this
      could be shown by repeating the proof in
      \citep{patternpaper}, which would show that $D^2L(\mathbf{W})$ induces a
      positive definite quadratic form on $Z_1$. We will briefly outline
      this argument. Due to the explicit formula for $D^2L(\mathbf{W})$
      (equation~\eqref{eq:second-derivative}) we have
      \[ D^2L(\mathbf{W})(\mathbf{U},\mathbf{U}) = \sum_{n=1}^{N}  
        \left(\mathbf{U}\,\mathbf{x}^{(n)}\right)^\intercal \mathbf{Q}^{(n)} 
        \left(\mathbf{U}\,\mathbf{x}^{(n)}\right). \]              
      In \citep{patternpaper} it was shown that $\mathbf{Q}^{(n)}$ is
      a non-negative definite matrix, with a simple eigenvalue $0$
      with eigenvector $\ones$. Hence, $D^2L(\mathbf{W})(\mathbf{U},\mathbf{U}) > 0$
      unless
      \[\mathbf{U}\,\mathbf{x}^{(n)}=c_n\ones,\quad n=1,2,\ldots,N.\]
      Equivalently,
      $\mathbf{U}\mathbf{X} = \ones\mathbf{c}^\intercal\in\Gamma_1$, i.e.,
      $\mathbf{U}\in\mathbf{F}_X^{-1}(\Gamma_1)$. If we additionally
      assume that $\mathbf{U}\in Z$ then
      $\mathbf{U}\in \mathbf{F}_X^{-1}(\Gamma_1)\cap Z \subseteq 
      \Nullspace(\mathbf{F}_X)\cap
      Z = Z_0$. As $Z_1$ is a complement of $Z_0$ in $Z$,
      $D^2L(\mathbf{W})(\mathbf{U},\mathbf{U}) \neq 0$ (equivalently, $>0$)
      for all $\mathbf{U}\in Z_1$. This demonstrates that $L|Z_1$ is
      a locally strongly convex function, and yields the conclusion of the 
      proof.
    \end{proof}

    There is also an alternative proof, which we will present
    here, using a coordinate-dependent style of
    argument. Conceptually, this proof reduces the proof of
    Theorem~\ref{thm:existence-of-minimum-arbitrary-rank} to
    applying Theorem~\ref{thm:existence-of-minimum-rank-D}.  This
    proof has an additional value, as it has a practical approach
    to finding the minimum of $L$ and reducing the number of
    weights.
    \begin{proof}
      (An alternative proof of Theorem~\ref{thm:existence-of-minimum-arbitrary-rank})
      Let $K = \rank{\mathbf{X}}$. Let us identify $\mathbb{R}^D$ with
      $\mathbb{R}^K\oplus\mathbb{R}^{D-K}$ where $\mathbb{R}^K$ is embedded
      into $\mathbb{R}^D$ as the first $K$ coordinates and
      $\mathbb{R}^{D-K}$ as the last $D-K$ coordinates.  There is an
      invertible matrix
      $\mathbf{S}\in L(\mathbb{R}^D,\mathbb{R}^D)$ such that
      $\Range(\mathbf{S}\mathbf{X}) = \mathbb{R}^K \subseteq \mathbb{R}^D$. Then we
      have the following obvious change of variables formula:
      \[ L(\mathbf{W};\mathbf{X}) = L(\widetildevec{W}; \widetildevec{X}). \]
      where $\widetildevec{W}=\mathbf{W}\mathbf{S}^{-1}$ and
      $\widetildevec{X} = \mathbf{S}\mathbf{X}$.  Therefore the minima
      of $L(\cdot;\mathbf{X})$ and $L(\cdot;\widetildevec{X})$ are
      the same up to a linear change of variables in the space of
      weight matrices. In particular, the global minima correspond
      and strong convexity is preserved. Furthermore, we consider
      the following partitions of $\widetildevec{X}$ and $\widetildevec{W}$
      into submatrices,
      \[ \widetildevec{X} =\left[
          \begin{array}{c}
            \doublewidetildevec{X}\\
            \hline
            \mathbf{0}
          \end{array}
        \right]
        \quad\text{and}\quad
        \widetildevec{W} = \left[
          \begin{array}{c|c}
            \doublewidetildevec{W} & \mathbf{*}
          \end{array}
        \right],
      \]
      where $\doublewidetildevec{X}$ is a submatrix of
      $\widetildevec{X}$ consisting of the first $K$ rows of
      $\widetildevec{X}$ (a $K\times N$ matrix of rank $K$), $\mathbf{0}$
      is a $(D-K)\times N$ matrix of zeros, and
      where $\doublewidetildevec{W}$ is submatrix of
      $\widetildevec{W}$ consisting of the first $K$ columns of
      $\widetildevec{W}$, while $\mathbf{*}$ is a ``wildcard''
      submatrix consisting of the last $D-K$ columns of
      $\widetildevec{W}$.  The expression
      $L\left(\widetildevec{W};\widetildevec{X}\right)$ does not
      explictly depend on the last $D-K$ columns of the matrix
      $\widetildevec{W}$ and therefore
      \[L(\widetildevec{W};\widetildevec{X})=L(\doublewidetildevec{W};\doublewidetildevec{X}).\]
      Also,
      $L(\mathbf{W};\mathbf{X}) =
      L(\doublewidetildevec{W};\doublewidetildevec{X})$ and there
      is correspondence between the global minima of
      $L(\cdot;\mathbf{X})$ and $L(\cdot;\doublewidetildevec{X})$,
      but it is not a 1:1 correspondence because many matrices
      $\mathbf{W}$ correspond to a single matrix
      $\doublewidetildevec{W}$ by varying the last $D-K$ columns
      of $\widetildevec{W}=\mathbf{W}\mathbf{S}^{-1}$.

      By Theorem~\ref{thm:existence-of-minimum-rank-D},
      $L(\cdot;\doublewidetildevec{X})$ is locally strongly convex on
      $Z\subseteq L(\mathbb{R}^K,\mathbb{R}^C)$ and has a unique global
      minimum (say, $\doublewidetildevec{W}$) there. (Note that
      the domain of $L(\cdot;\doublewidetildevec{X})$ is
      $L(\mathbb{R}^K,\mathbb{R}^C)$, so $Z$ denotes a different vector
      space than in prior discussion.) All global minima of
      $L(\cdot;\widetildevec{X})$ are obtained by varrying the
      wildcard portion of $\widetildevec{W}$.  Finally, we
      identify the wildcard portion as
      $\Nullspace(\mathbf{F}_{\widetildevec{X}})$ and the change of
      coordinates
      $\mathbf{W}\mapsto\widetildevec{W}=\mathbf{W}\mathbf{S}^{-1}$ maps bijectively
      $\mathbf{W}\in \Nullspace(\mathbf{F}_{\mathbf{X}})$ to 
      $\widetildevec{W}\in \Nullspace(\mathbf{F}_{\widetildevec{X}})$. It also 
      maps $Z$ bijectively onto itself. These observations
      imply all statements of the theorem.
    \end{proof}

\section{Convergence rate under gradient descent}
    \label{sec:rate-of-convergence}
One can define
a continuous map based upon \eqref{eq:grad-descent}, as
\begin{equation}
\boldsymbol\Phi(\mathbf{W}) = \mathbf{W} - \eta\,\boldsymbol\nabla L(\mathbf{W}). 
\end{equation}
Then Taylor expanding $\boldsymbol\Phi(\mathbf{W})$ about the global
minimizer $\vecwidehat{W}$, we can obtain approximate expressions for
$\boldsymbol\Phi(\mathbf{W}^{(n+1)})$ and
$\boldsymbol\Phi(\mathbf{W}^{(n)})$. This leads to the following error estimate:
\begin{equation}
 \mathbf{e}^{(n+1)} \approx D\boldsymbol\Phi(\vecwidehat{W}) \mathbf{e}^{(n)}
\end{equation}
where $\mathbf{e}^{(n+1)}$ is the error for iteration number ${(n+1)}$. 
We
want to bound
$\norm{D\boldsymbol\Phi(\vecwidehat{W})}_2= 
\norm{\mathbf{I}-\mathbf{H}(\vecwidehat{W}) }_2$
and of course ensure that it
is less than 1 so that we have a contraction mapping.  We thus require $\norm{\mathbf{I}-\mathbf{H}(\vecwidehat{W}) }_2$
must be within the interval $[-\theta,\theta]$, where
$\theta\in(0,1)$. Let $\lambda_{max}, \lambda_{min}$ be the largest
and smallest eigenvalues of $\mathbf{H}$. As noted in \citep{patternpaper},
\begin{align*}
-\theta & \le 1 - \eta\,\lambda_{max}\\
\theta & \ge 1 - \eta\,\lambda_{min}.
\end{align*}
This implies that for a contraction, it is necessary to have 
\[ K \le \frac{1+\theta}{1-\theta} \]
where
\[ K = \frac{\lambda_{max}}{\lambda_{min}} \]
is also the \textbf{condition number} of the Hessian. Before we provide bounds on $\lambda_{min}, \lambda_{max}$ we first derive the Hessian for the MLE.

\section{Hessian for MLE}
In the previous section it was shown that we need the largest and smallest eigenvalues of the Hessian. In this section, we obtain the Hessian. Working with the Hessian for the space $Z$ is
difficult so instead we work with the isomorphic space
$L(\mathbb{R}^{C-1},\mathbb{R}^N)$ via a mapping
$L(\mathbb{R}^{C-1},\mathbb{R}^N) \to Z$ 
given by $\mathbf{S} \mapsto \mathbf{KS}$ where $\mathbf{K}$ has dimension $C$ by
$C-1$ and $\mathbf{S}$ has dimension $C-1$ by $N$, and where the column means 
are 0. The mapping is invertible, so we have 
$\mathbf{KS} \ni Z \to  \mathbf{S} \in L(\mathbb{R}^{C-1},\mathbb{R}^N)$. 
One choice for $\mathbf{K}$ is
\begin{equation}
\label{nonisometric-K}
\mathbf{K} = \begin{bmatrix} 
1 & 0 &  0 &\dots 0 \\
0 & 1 &  0 & \dots 0 \\
0 & 0 &  1 & \dots 0 \\
&   & \ddots  &   \\
0 & 0 &  0 & 1 \\
-1  & -1   & -1 & -1
\end{bmatrix}.
\end{equation}
Another choice is to choose $\mathbf{K}$ so that the mapping  $\mathbf{KS} \to 
\mathbf{S}$ is an isometry, i.e.
\begin{equation}
\label{isometric-K}
\mathbf{K}^T\mathbf{K}=\mathbf{I}_{C-1}.
\end{equation}
It can be
easily shown that if the mapping is an isometry then the induced Hessian on 
the space $L(\mathbb{R}^{C-1},\mathbb{R}^N)$ has the same eigenvalues
as the Hessian defined on the space $Z$. Thus we can study the 
eigenvalues of 
our problem by working on the space $L(\mathbb{R}^{C-1},\mathbb{R}^N)$ 
via the isometry \eqref{isometric-K}. This is the first motivation for studying the space 
$L(\mathbb{R}^{C-1},\mathbb{R}^N)$.  A second motivation is that one may 
want to do gradient descent using the space 
$L(\mathbb{R}^{C-1},\mathbb{R}^N)$ instead of $Z$. One would then 
need to study the eigenvalues of the induced Hessian on 	
$L(\mathbb{R}^{C-1},\mathbb{R}^N)$ to obtain convergence properties. The 
eigenvalues will vary depending 
on the type of $\mathbf{K}$ used which we shall see later.

For the moment, assume $\vec{X}$ is in $L(\mathbbm{R}^N,\mathbbm{R}^N)$ (the general 
$D\times N$ case will be dealt with later). For
simplicity, we just consider one sample and so drop superscripts $(n)$
from the vectors $\mathbf{x}, \mathbf{t}, \mathbf{y}$. Then, using the Chain
rule for {Fr\'echet derivatives}, and noting that $\mathbf{W}$=$\mathbf{K}\mathbf{S}$
for some $\mathbf{S} \in L(\mathbb{R}^{C-1},\mathbb{R}^N) $,
\begin{equation}
\label{eqn:truncated-gradient}
DL(\mathbf{S}) (\mathbf{P}) =  DL(\mathbf{W})\circ D\mathbf{W} (\mathbf{S}) \mathbf{P} = 
-(\mathbf{t}-\mathbf{y})^\intercal\,\mathbf{K P}\,\mathbf{x}.  
\end{equation}
where $\mathbf{P}$ has the same dimensions as $\mathbf{S}$.  From the
definition of gradient, one has
\begin{equation*}
DL(\mathbf{S}) (\mathbf{P}) = \langle \nabla L(\mathbf{S}), \mathbf{P} \rangle
\end{equation*}
where $\langle \cdot, \cdot \rangle = \tr(\cdot, \cdot)$. It follows
\begin{equation*}
-(\mathbf{t}-\mathbf{y})^\intercal\,\mathbf{K P}\,\mathbf{x} = \langle \nabla L(\mathbf{S}), 
\mathbf{P} \rangle.
\end{equation*}
Using the fact that
$\mathbf{a^\intercal} \mathbf{b}=\tr(\mathbf{a^\intercal} \mathbf{b})=\tr(\mathbf{b} 
\mathbf{a^\intercal})$, the
previous equation becomes
\begin{equation*}
\tr (  -\mathbf{K P}\,\mathbf{x}   (\mathbf{t}-\mathbf{y})^\intercal\ ) = \tr (  - 
\mathbf{x}   (\mathbf{t}-\mathbf{y})^\intercal\ \mathbf{K P} ) = \langle \nabla 
L(\mathbf{S}), \mathbf{P} \rangle.
\end{equation*}
From this it follows that 
\begin{equation*}
\nabla L(\mathbf{S})= - \mathbf{x}   (\mathbf{t}-\mathbf{y})^\intercal\ \mathbf{K}. 
\end{equation*}
Let $\mathbf{R} \to \mathbf{K} \mathbf{R}$ where
$\mathbf{R} \in L(\mathbb{R}^{C-1},\mathbb{R}^N)$ and $\mathbf{K}$ is the
same as before. Then similarly, we can write the second derivative in
terms of $\mathbf{S}$, acting in the directions of $\mathbf{P}$ and
$\mathbf{R}$ as
\begin{equation*}
D^2L(\mathbf{S}) (\mathbf{P}, \mathbf{R}) = \mathbf{x}^\intercal (\mathbf{K} \mathbf{R} 
)^\intercal \mathbf{Q} \mathbf{K} \mathbf{P} \mathbf{x}.
\end{equation*}
We then use the definition of Hessian to write 
\begin{equation*}
D^2L(\mathbf{S}) (\mathbf{P},\mathbf{R}) = \langle \mathbf{H}(\mathbf{R}), \mathbf{P}   
\rangle
\end{equation*}
from which we may write
\begin{equation*}
\langle  \mathbf{x} \mathbf{x}^\intercal (\mathbf{K}\mathbf{R})^\intercal   \mathbf{Q} \mathbf{K}, 
\mathbf{P}  \rangle =  \langle \mathbf{H}(\mathbf{R}), \mathbf{P}   \rangle
\end{equation*}
which gives 
\begin{equation*}
\mathbf{x} \mathbf{x}^\intercal (\mathbf{K}\mathbf{R})^\intercal    \mathbf{Q}\mathbf{K}   = 
\mathbf{H}(\mathbf{R}).
\end{equation*}
In order to represent $\mathbf{H}$ as a matrix, we apply the operator $\stack$ to both sides of the above equation to get
\begin{equation*}
\stack ( \mathbf{H}(\mathbf{R}))=\left( ( \mathbf{K^T Q K}) \otimes (\mathbf{x} 
\mathbf{x}^\intercal)  \right)  \stack   (\mathbf{R}^\intercal)
\end{equation*}
where $\stack$ takes a matrix and outputs its columns stacked. Considering contributions from each $\mathbf{x}^{(n)}$, 
and re-defining
$\mathbf{H}$ as the matrix acting on $\stack(\mathbf{R})$ we have
\begin{equation}
\mathbf{H}= \sum_{n=1}^{N} \mathbf{H}^{(n)} 
\end{equation}
where
\begin{equation}
\label{eq:hess}
\mathbf{H}^{(n)}= \mathbf{A}^{(n)} \otimes \mathbf{B}^{(n)} 
\end{equation}
and
\begin{eqnarray}
 \mathbf{A}^{(n)}= \mathbf{K}^\intercal \mathbf{Q}^{(n)} \mathbf{K} \\
\mathbf{B}^{(n)}=\mathbf{x}^{(n)} {\mathbf{x}^{(n)}}^\intercal
\end{eqnarray}
and the superscript $(n)$ indicates quantities corresponding the the $n^{th}$
sample. Note that $\mathbf{H}$ is in $L(\mathbb{R}^{N(C-1)},\mathbb{R}^{N(C-1)} )$.  

\section{Eigenvalue bounds when $N=D$}
We provide bounds on the eigenvalues (from below and above) of $\mathbf{H}$ which are necessary
for our investigation into convergence rate. We first consider the case $N=D$.
\begin{lemma} \label{lemma:direct_sum}
    $\mathbb{R}^{N(C-1)}= \bigoplus_{n=1}^{n=N} \Range(\mathbf{H}^{(n)})$
\end{lemma}
\begin{proof}
    Each $\mathbf{H}^{(n)}$ has rank $C-1$ by the rank property of Kronecker 
    products. We also know that the rank of $\mathbf{H}$ is $N(C-1)$ as it is 
    invertible. This implies $\Range(\mathbf{H}^{(i)}) \bigcap 
    \Range(\mathbf{H}^{(j)})=\emptyset $ for $i \neq j$. We thus have 
    that $\mathbb{R}^{N(C-1)}= \bigoplus_{n=1}^{n=N} 
    \Range(\mathbf{H}^{(n)})$.
\end{proof}
Let $\lambda_i(\mathbf{G})$ denote the $i^{th}$ largest eigenvalue of 
$\mathbf{G}$, for some Hermitian matrix $\mathbf{G}$, where eigenvalues are 
reapeated according to their (algebraic) multiplicity.

\begin{lemma}[Weyl \citep{Weyl}]
    Let $\mathbf{A}$ and $\mathbf{B}$ (not to be confused with our earlier 
    definitions of $\mathbf{A}^{(n)}$ and $\mathbf{B}^{(n)}$ in \eqref{eq:hess}) 
    be Hermitian and in $L(\mathbb{R}^N,\mathbb{R}^N)$ with eigenvalues
    $(\alpha_i)_{i=1}^{N}$ and $(\beta_i)_{i=1}^{N}$ sorted in
    decreasing order. Then
    \begin{equation}
    \label{weyl1}
    \alpha_i+ \beta_N  \leq \lambda_i(\mathbf{A}+\mathbf{B}) \leq \alpha_i+ 
    \beta_i	
    \end{equation}		
\end{lemma}

\begin{corollary}[Weyl Perturbation]
    Let $\vec{A}$ and $\vec{B}$ be any two Hermitian operators in 
    $L(\mathbb{R}^N,\mathbb{R}^N)$ with eigenvalues of $\vec{A}$ given by 
    $(\alpha_i)_{i=1}^{N}$, in decreasing order. Then
    \begin{equation}
    \label{weyl2}
    \alpha_i-\norm{\vec{B}}_2  \leq \lambda_i(\vec{A}+\vec{B}) \leq 
    \alpha_i+ \norm{\vec{B}}_2	
    \end{equation}
\end{corollary}

\begin{corollary}\label{cor:eigbound1}
    \begin{equation}
    \label{lambdaNupper}
    \lambda_{N(C-1)}(\vec{H}) \leq C  \min_{1 \leq i \leq 
        N}{\norm{\vec{x}^{(i)} }_2}
    \end{equation}
\end{corollary}

\begin{proof}
    Let $\vec{A}= \sum_{n=1}^{N-1} \vec{H}^{(n)}, \ \vec{B}=\vec{H}^{(N)}$. 
    Then applying \eqref{weyl2}, it follows 
    \begin{equation*}
    \left| \lambda_{N(C-1)}(\vec{H}) - \lambda_{N(C-1)}\left( 
        \sum_{n=1}^{N-1} \vec{H}^{(n)} \right) \right | \leq  \norm{ 
        \vec{H}^{(N)}}_2.
    \end{equation*}
    Since  $\sum_{n=1}^{N-1} \vec{H}^{(n)}$ is not full rank,  
    $\lambda_{N(C-1)} \left( \sum_{n=1}^{N-1} \vec{H}^{(n)} \right)=0$. 
    Also,
    \begin{equation*}
    \norm{ \vec{H}^{(N)}}_2  =  \norm{ \lambda_1(\vec{A}^{(N)})}_2 \norm{ 
        \lambda_1(\vec{B}^{(N)})}_2= \norm{\lambda_1(\vec{A}^{(N)}) }_2 
        \norm{ 
        \vec{x}^{(N)} }_2^2 
    \end{equation*} 
    where $\vec{A}^{(N)}$ and $\vec{B}^{(N)}$ are the same as in
    \eqref{eq:hess}, and where we used the formula for eigenvalues of
    Kronecker products.  Since the choice of $\vec{H}^{(N)}$ is
    arbitrary, we take $\min_{1 \leq i \leq N}{\| \vec{x}^{(i)} \|_2}$.
    We provide an upper bound for $\norm{ \lambda_1(\vec{A}^{(N)}) }_2$
    as follows. The fact that $\vec{A}^{(N)}$ is positive definite and
    $ \sum_{i=1}^{N} \lambda_i \left(\vec{A}^{(N)} \right) = \tr 
    \vec{A}^{(N)} $
    gives

\begin{equation*}
	\begin{split}
		\lambda_1\left(\vec{A}^{(N)}\right)  &\leq  \tr{\vec{A}^{(N)}} \\
    &=\tr(\vec{Q} \vec{K} \vec{K}^\intercal) \\ 
    &\leq \sqrt{ \langle \vec{Q}, \vec{Q} \rangle }
    \sqrt{ \langle \vec{K} \vec{K}^\intercal , \vec{K} \vec{K}^\intercal  \rangle 
    } \\ 
    &\leq \sqrt{C} \sqrt{ \tr(\vec{K}^\intercal \vec{K} \vec{K}^\intercal \vec{K} 
        )} < C
	\end{split}
\end{equation*}
    where we use \eqref{isometric-K}.
\end{proof}

\begin{corollary} \label{cor:eigbound2}
\begin{equation*}
\begin{split}
    &\max
    \left \{\left( \min_{n,i}  y_i^{(n)}   \right)
    \left( \max_i \norm{\vec{x}^{(i)} }_2 \right),
    \left( \max_{n,i}  y_i^{(n)}   \right)
    \left( \min_i  \norm{\vec{x}^{(i)} }_2 \right)     \right \}\\ 
    &\leq \lambda_1 \left(\vec{H} \right) \\
    &\leq  C {\norm{\vec{X} }_F}
\end{split}
\end{equation*}
\end{corollary}

\begin{proof}
    The inequality on the right follows from the triangle inequality
    applied to $\sum_{n=1}^{N} \vec{H}^{(n)}$. For the inequality on the
    left, similarly as before, let $\vec{A} =\vec{H}^{(i)}$ and
    $\vec{B}= \vec{H}-\vec{H}^{(i)}$ for some $1 \leq i \leq N $, then
    apply \eqref{weyl1} using the fact that $\vec{A}$ is positive
    semi-definite and singular to get
    $ \max_{1 \leq i \leq N} \lambda_1 \left( \vec{H}^{(i)} \right) \leq
    \lambda_1(\vec{H})$. Note

    \begin{equation*}
    \begin{split}
        {\max_i \lambda_1 \left( \vec{H}^{\left(i\right)}   \right) } &= \max_i 
    \left( \lambda_1 \left(\vec{A}^{(i)} \right) \  \norm{\vec{x}^{(i)} }_2 
    \right)  \\ 
        &\geq \left( \max_{1 \leq i \leq N} 
    \lambda_1(\vec{A}^{(i)})\right) \left( \min_{1 \leq i \leq N} \norm 
    {\vec{x}^{(i)} }_2 \right).
    \end{split}
    \end{equation*}
Similarly,    
    \begin{equation*}
    {\max_i \lambda_1 \left( \vec{H}^{\left(i\right)}   \right) }  \geq    
    \left( \min_{1 \leq i \leq N} \lambda_1(\vec{A}^{(i)})\right) \left( 
    \max_{1 \leq i \leq N} \norm{\vec{x}^{(i)} }_2 \right).  
    \end{equation*}		
    To bound $\lambda_1(\vec{A}^{(i)})$ from below we use
    
    \begin{equation*}
    \lambda_{1}(\vec{A}^{(i)})   = \sup_{\vec{u}} \frac{(\vec{u} 
            \vec{K})^\intercal \vec{Q}^{(i)} \vec{Ku}}{\lVert \vec{u} 
            \rVert^2} = 
    \sup_{\vec{u}} 
    \frac{(\vec{u} 
        \vec{K})^\intercal \vec{Q}^{(i)} \vec{Ku}}{\lVert \vec{Ku}  \rVert^2}  
            = \sup_{\vec{K} \vec{u}} 
            \frac{(\vec{u} 
                \vec{K})^\intercal \vec{Q}^{(i)} \vec{Ku}}{\lVert \vec{Ku}  
                \rVert^2} 
    \end{equation*}
    where we use \eqref{nonisometric-K}. From \citep{patternpaper} we know 
    that $\ones$ is the
    only eigenvector of $\vec{Q}$ with eigenvalue $0$,
    which implies
    $\left\langle\vec{q}^{(i)},\ones\right\rangle=0$,
    $i=1,2,\ldots C-1$ where $\vec{q}^{(i)}$ is an
    eigenvector of $\vec{Q}$ corresponding to the $i^{th}$
    largest eigenvalue of $\vec{Q}$. From this it follows that		
    \begin{equation}
    \label{eq:K-span}
    \mathbb{R}^C \setminus\{\ones\}=\Span\left\{ \vec{ 
        q}^{(i)}\right\}_{i=1}^{C-1}=  
    \Span\left\{ \vec{k}^{(i)}\right\}_{i=1}^{C-1}.
    \end{equation}
    For $\vec{K}$ given by \eqref{nonisometric-K} and \eqref{isometric-K}.  
    Also, recalling from \citep{patternpaper} that
    $\lambda_{1}(\vec{Q}^{(n)}) \geq \max_i {y_i^{(n)}}$,
    we have 
    
    \begin{equation*}
    \lambda_{1}(\vec{A}^{(i)})   \geq	\max_{i} y_i^{(i)}.
    \end{equation*}
 Combining our results,
    \[	
    \lambda_1\left( \vec{H}^{\left(i\right)}   \right) \geq \max \left 
    \{     
    \left( \min_{n,i}    y_i^{(n)}   \right)   \left( \max_i 
    \norm{\vec{x}^{(i)} }_2 \right), \left( \max_{n,i}  y_i^{(n)}   
    \right)   \left( \min_i  
    \norm{\vec{x}^{(i)} }_2 \right) \right \}
    \].
\end{proof}

\begin{lemma}
    \begin{equation*}\label{lemma:lambda_N_better}
    \lambda_{N(C-1)}\left(\mathbf{H}\right) \geq   \left( \min_{n,i} \norm{ 
        \mathbf{x}^{(n)}}_2^2    y_i^{(n)} \right).
    \end{equation*}
\end{lemma}
	
 \begin{proof}
    For Hermitian matrices, the minimum of the Rayleigh quotient gives the 
    smallest eigenvalue: 
    \begin{equation*}
    \lambda_{N(C-1)}\left(\mathbf{H}\right) 
    =\min_{ \norm{u}_2=1} \langle u,\mathbf{H}u\rangle 
    = \min_{ \norm{u}_2=1} \sum_{n=1}^{N} \langle u,\mathbf{H}^{(n)}u\rangle.  
    \end{equation*}
    By Lemma~\ref{lemma:direct_sum}, write $\mathbf{u}= \sum_{n=1}^{N} 
    \mathbf{u}^{(n)}$ where $\mathbf{u}^{(n)} \in \Range(\mathbf{H}^{(n)})$. 
 Using these facts we get
    \begin{equation}
    \label{eq:incomplete-ineq}
    \min_{ \|u\|_2=1} \sum_{n=1}^{N} \langle u,\mathbf{H}^{(n)}u\rangle  \geq  
    \left(\sum_{n=1}^{N} \min_{u^{(n)} \in \Range(\mathbf{H}^{(n)})\setminus 
    \mathbf{0} } 
    \langle u^{(n)},\mathbf{H}^{(n)}u^{(n)}\rangle \right)\bigg 
    |_{\norm{u}_2=1}.      
    \end{equation}	
    Now use $1=\lVert \mathbf{u} \rVert_2 \leq \sum_{n=1}^{N} \norm{ 
        \mathbf{u}^{(n)} 
    }_2$ along with the fact that $\lambda_{N(C-1)}(\mathbf{H}^{(n)}) = 0$ for 
    all $n$, to find that the right hand side of \eqref{eq:incomplete-ineq} 
    is greater or equal to
    \begin{equation*}
    \min_n \min_{  u \in \Range(\mathbf{H}^{(n)}) \setminus 
    \mathbf{0},\;\norm{u}_2=1} \left\langle 
    u,\mathbf{H}^{(n)}u\right\rangle.
    \end{equation*}	  
    Let $\mathbf{H}^{(n)}$ have spectral resolution $\sum_{j \in S} 
    \alpha_j^{(n)} \mathbf{q}_j^{(n)} \left( \mathbf{q}_j^{(n)}\right) ^T$, where 
    $S=\{j: \alpha_j^{(n)} \neq0 \}$. Then
    \begin{equation*}
              \min_{  \mathbf{u} \in \Range(\mathbf{H}^{(n)}),\ \norm{u}_2=1  } \left\langle 
                \mathbf{u},\mathbf{H}^{(n)}\mathbf{u}\right\rangle
              = \min_i \left\{\lambda_i \left(\mathbf{H}^{(n)} \right)\,:\, 
              \lambda_i \left(\mathbf{H}^{(n)} \right)\neq 0\right\}. 
    \end{equation*}	
    Also, by the eigenvalue property of Kronecker products,
    \begin{equation}
    \label{eq:lambda_N_kronecker}
    \min_i \left\{ \lambda_i\left(\mathbf{H}^{(n)}\right) : 
    \lambda_i\left(\mathbf{H}^{(n)}\right) \neq 
    0\right\} 
    = 
    \lambda_{C-1}\left(\mathbf{A}^{(n)}\right)  \cdot \lVert\mathbf{x}^{(n)}\rVert_2^2.
    \end{equation}
    For the case in which $\mathbf{K}$ is given by \eqref{isometric-K}, we 
    have 	
    \begin{equation*}
    \lambda_{C-1}\left(\mathbf{A}^{(n)}\right)   = \inf_{\mathbf{u}} 
    \frac{(\mathbf{u}\mathbf{K})^\intercal \mathbf{Q}^{(n)}\mathbf{K} \mathbf{u}}{\norm{ 
    \mathbf{u}}_2^2}  = 
            \inf_{\mathbf{u}} 
    \frac{(\mathbf{u}\mathbf{K})^\intercal \mathbf{Q}^{(n)}\mathbf{K} 
    \mathbf{u}}{\norm{\mathbf{Ku}}_2^2}=\inf_{ \mathbf{K} \mathbf{u} } 
        \frac{\mathbf{ 
        (uK)^\intercal} \mathbf{Q}^{(n)}\mathbf{Ku}}{\norm{\mathbf{Ku}}_2^2}.
    \end{equation*}		
    where we use \eqref{eq:K-span}. Thus
    \begin{equation*}
    \lambda_{C-1}\left(\mathbf{A}^{(n)}\right)  = \min_{i} y_i^{(n)}.
    \end{equation*}		
    Combining this with \eqref{eq:lambda_N_kronecker} we get 
    Lemma~\ref{lemma:lambda_N_better}. We get a similar result for the case of 
    $\mathbf{K}$ given by \eqref{nonisometric-K}:
    \begin{equation*}
              \lambda_{C-1}\left(\mathbf{A}^{(n)}\right)
              = \inf_{\mathbf{u}} \frac{ (\mathbf{u} \mathbf{K} )^\intercal}{ \mathbf{Q}^{(n)}\mathbf{K}\mathbf{u} \norm{ \mathbf{u} } _2^2 } \geq 
            \inf_{\mathbf{u}} \frac{ (\mathbf{u}\mathbf{K})^\intercal 
            \mathbf{Q}^{(n)}\mathbf{K}\mathbf{u} }{ \norm{\mathbf{K}\mathbf{u}}_2^2 } 
            =\min_{i} y_i^{(n)}.
        \end{equation*}.
\end{proof}	
We showed that isometric and non-isometric $\vec{K}$ give the same
lower bound, i.e.,
$\lambda_{C-1} \left(\vec{A}^{(n)}\right) \geq \min_{i} y_i^{(n)}$. When
$\vec{K}$ given by \eqref{nonisometric-K} we get an inequality, but we have 
an equality	when $\vec{K}$ is given by $\eqref{isometric-K}$. In either case, we have a nonzero lower bound. The figures below show how $\lambda_{C-1} 
\left(\vec{A}^{(n)}\right)$ behaves when $\vec{K}$ is given by
\eqref{nonisometric-K} and when $y_i^{(n)}$ have the same marginal
distribution. For each realization, $\{y_i\}$ are chosen from a uniform 
distribution on $(0,1)$ then divided by	$\sum_{i=1}^{C} y_i$ so that their 
sum is 1.

\begin{figure}[ht]
\vskip 0.2in
  \begin{minipage}{.48\linewidth}
    \begin{center}
    \centering
    \includegraphics[width=\textwidth]{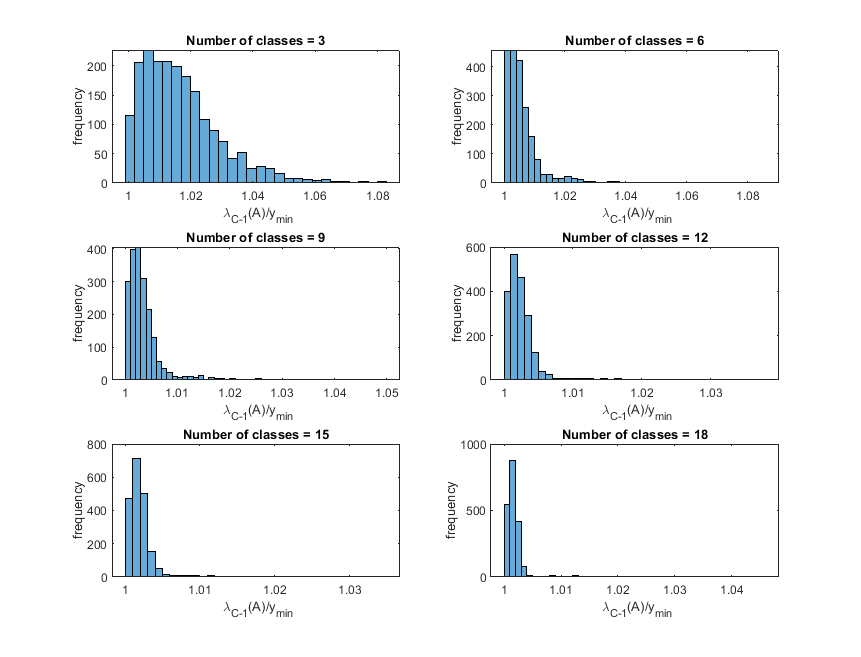}
    \end{center}
\end{minipage}
  \begin{minipage}{.48\linewidth}
    \begin{center}
    \centering
    \includegraphics[width=\textwidth]{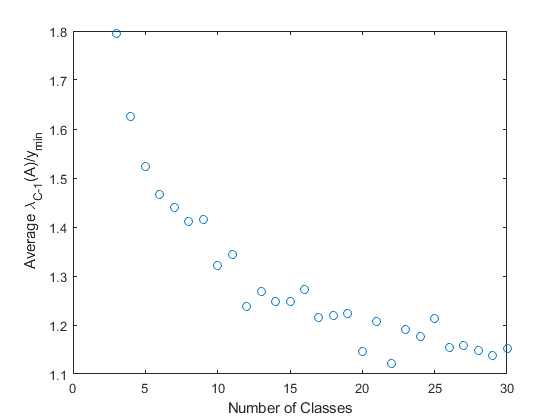}
    \end{center}
\end{minipage}
\caption{Each figure corresponds to 2000 realizations, with the matrix $\mathbf{K}$ used to compute $\mathbf{A}$ given by
    \eqref{nonisometric-K}.  The eigenvalues of $\mathbf{A}$ are computed for $C=3,6,9,12,15,18$. On left: plots of the frequencies of the
    ratio $\frac{\lambda_{C-1}\left(A\right)}{y_{min}}$.  On right: plots of the frequencies of the ratio   $\frac{\lambda_{C-1}\left(A\right)}{y_{min}}$. }
\label{fig:1}
\vskip 0.3in
\end{figure}
From the figures in FIG~\ref{fig:1}, one see's that $\lambda_{C-1} \left(\vec{A}^{(n)}\right) \to 
1 $ as
$C \to \infty$ if $\{y_i\}$ have the same marginal distribution.  This
begs the question: is this behavior correct, and independent of the
way $\{y_i\}$ are generated? The answer is yes, and we show this by first analyzing the expression
\[\inf_{\vec{u}: \lVert \vec{u} \rVert_2=1} =
\frac{ (\vec{u}\vec{K})^\intercal \vec{Q}^{(n)}\vec{K}\vec{u}}{\norm{ 
        \vec{u}}_2^2}
\]	 
which we know is
bounded below by $\min_{i} y_i^{(n)}$.  Using the definition of
Rayleigh quotient, the figures suggest that there typically exists
$\vec{u}$ such that $\lVert \vec{u} \rVert_2=1$ and
$ \lVert \vec{ Ku}-\vec{q}_{C-1}\rVert_2$ is small, and that this
approximation gets better as $C \to \infty$, where ${\vec{q}}_{C-1}$
is the eigenvector corresponding to the $(C-1)^{st}$ largest
eigenvalue of $\vec{Q}^{(n)}$.  Note
$\norm{ \vec{K}\vec{u}-\vec{q}_{C-1} }_2$ is small if
$\norm {\vec{u}-\vec{\widetilde{q}}_{C-1}}_2$ is small, where
$\vec{\widetilde{q}}_{C-1}$ is equal to the first $C-1$ components of
$\vec{q}_{C-1}$.  Now realize that $\vec{q}_{C-1}$ (as are all
eigenvectors corresponding to non-zero eigenvalue) is in the range of
$\vec{K}$ since the columns of $\vec{K}$ are orthogonal to the the
vectors of 1's. So we may write
$ \vec{q}_{C-1}=\sum_{i=1}^{C-1} \alpha_i \vec{K}^{(i)} $.  Since
$\norm{\vec{q}_{C-1}}_2=1$, it follows that
$\sum_{i=1}^{C-1} \alpha_i^2+ \left( \sum_{i=1}^{C-1} \alpha_i
\right)^2=1$. Let
\begin{equation*}
V= \left \{\boldsymbol{\alpha} \in \mathbb{R}^{C-1}: \sum_{i=1}^{C-1}  
\alpha_i^2+ 
\left( \sum_{i=1}^{C-1} \alpha_i \right)^2=1 \right \}.
\end{equation*}
Our discussion can be rephrased as follows: a sufficient condition for
which $\lambda_{C-1} \left(\vec{A}^{(n)}\right)$ goes to 0 as $C \to \infty 
$ is:
$\min_{x \in S^{C-2}} dist(\boldsymbol{\alpha},\vec{x} ) \to 0$ as
$C \to \infty$ given any $\boldsymbol{\alpha} \in V$, where $S^{C-2}$
is the $(C-2)$-dimensional unit sphere. We formalize our discussion in
the following lemma:
\begin{lemma}
    Fix $n$ and assume each $y_i^{(n)}$ for $i=1,2,\ldots,C$ has the same
    marginal distribution. Then
    $\lambda_{C-1}\left(\vec{A}^{(n)}\right) \to \min_i 
    y_i^{(n)}$  in probability as $C \to \infty$
\end{lemma}

\begin{proof}
    From the preceding discussion, it suffices to show that \newline
    $\min_{x \in S^{C-2}} \dist(\boldsymbol{\alpha},\vec{u} )  \to 0$ as
    $C \to \infty$, for any $\boldsymbol{\alpha} \in V$. We do this by
    showing that
    $P\left(\left| \sum_{i=1}^{C-1} \alpha_i^2-1 \right| > \epsilon \right) 
    \to 0$
    as $C \to \infty$. For the proof, we drop the superscript $(n)$ from
    $\vec{A}^{(n)}$ and $y_i^{(n)}$ . Denote $q_i$ to be the $i^{th}$
    component of $\vec{q}_{C-1}$.  Since $\boldsymbol{\alpha}$ is the
    first $C-1$ components of $\vec{q}_{C-1}$, we can write

    \begin{equation*}
    	\begin{split}
    		 P\left(\left| \sum_{i=1}^{C-1} \alpha_i^2-1 \right| > \epsilon \right) 
    &= P\left( 1-\sum_{i=1}^{C-1} \alpha_i^2 > \epsilon \right)  + P \left( 
    \sum_{i=1}^{C-1} \alpha_i^2 - 1 > \epsilon \right)\\ 
    		&= P\left( 
    1-\sum_{i=1}^{C-1} q_i^2 > \epsilon \right).    
    	\end{split}
    \end{equation*}
By Chebychev's inequality 
    \begin{equation*}
    P\left(1-\sum_{i=1}^{C-1} q_i^2 > \epsilon \right) \epsilon \leq 
    \mathbb{E} \left(1-\sum_{i=1}^{C-1} q_i^2 \right)= \mathbb{E} \,q_C^2
    \end{equation*}	  	  
    where $q_i$ is the $i^{th}$ component of $\vec{q}_{C-1}$ and where
    we used the fact that $\sum_{i=1}^{C} q_i^2=1$. Note
    $\{q_i\}_{i=1}^{C}$ have the same distribution as can be seen from
    eigenvalue equation
    \begin{equation*}
    q_i=\frac{y_i \langle \vec{y}, \vec{q}_{N-1} \rangle }{y_i- 
        \lambda_{C-1}(\vec{y)}} 
    \end{equation*}
    which is derived in \citep{patternpaper}.  Thus,
    $\sum_{i=1}^{C} \mathbb{E} \, q_i^2=1$ implies
    $\mathbb{E} \, q_i^2=\frac{1}{C}$. Choose $C$ large such that
    $\frac{1}{C} < \epsilon^2$. Then the statement is proved.
\end{proof}

\section{Eigenvalue bounds when $N>D$}
The previous bounds can easily be generalized. Consider $N>D$. For some 
$\alpha$, let
$B_{\alpha}$ be a subset of elements of
$\{ \mathbf{H}^{(n)}\}$ such that the sum of elements in $B_{\alpha}$ is 
full rank and the sum of elements of the set
$\{ \mathbf{H}^{(n)}\} \setminus B_{\alpha}$ is not full rank. Denote the
set of all such $\alpha$ by $\mathcal{P}$. Then
Corollary~\ref{cor:eigbound1} becomes
\begin{corollary}
    $\lambda_{N(C-1)} (\mathbf{H}) \leq C \max_{\alpha \in \mathcal{P} }   
    \norm { \sum_{j \in \alpha} \mathbf{x}^{(j)}   }_2$
\end{corollary}
\begin{proof}
    Let $\mathbf{A}$ be equal to the sum of elements in $\{ \mathbf{H}^{(n)}\} 
    \setminus B_{\alpha}$ for some some
    $B_{\alpha}$. Then using the definition of
    $\mathbf{A}, B_{\alpha}$ and applying \eqref{weyl2}, we have that
    \begin{equation*}
    \left | \lambda_{N(C-1)}\left(\mathbf{H}\right) - 
    \lambda_{N(C-1)}\left(\mathbf{A} 
        \right) \right | \leq \norm {\sum_{j \in B_\alpha} \mathbf{H}^{(j)} 
        }_2.
    \end{equation*}	
    Since $\mathbf{A}$ is not full rank, $\lambda_{N(C-1)} \left(\mathbf{A} 
    \right)=0$.	
    In a similar manner to the proof of Corollary~\ref{cor:eigbound1} we can write
    \begin{equation*}
    \left |\lambda_{N(C-1)} \left(\mathbf{H}\right)  \right | \leq  C 
    \min_{\alpha 
    \in 
        \mathcal{P} }   \norm { \sum_{j \in \alpha} \mathbf{x}^{(j)}   }_2. 
    \end{equation*}
\end{proof} 	
Corollary~\ref{cor:eigbound2} remains the same.
	
\section{Bounds on condition number}
For some $\gamma$, let $A_\gamma$ be a subset of elements of
$\{ \mathbf{H}^{(n)}\}$, such that its cardinality is $D$, and the sum of 
elements in $A_{\gamma}$ is 
full rank. Denote the set of all such $\gamma$ as $\mathcal{O}$.  From our 
eigenvalue bounds and using the definition of condition number, we have: 
\begin{theorem}
    When $X$ has dimensions $N=D$,
    \footnotesize{
    \begin{align*}
    \kappa(\mathbf{H}) &\leq \frac{ C \norm{\mathbf{X}}_F}{ \min_{n,i} \left( \norm{ 
            \mathbf{x}^{(n)}}_2^2    y_i^{(n)} \right) },\\
    \kappa(\mathbf{H}) &\geq   \frac{	\min_{n} \left( 
        \norm{\mathbf{x}^{(n)}}_2^2 \right) \  \max \left \{\left( \min_n  
        \max_i  
        y_i^{(n)}   \right)    \max_i \norm{\mathbf{x}^{(i)} }_2   ,   \left( 
        \max_{n,i} y_i^{(n)}   \right)   \min_i  \norm{\mathbf{x}^{(i)} 
        }_2     \right \}   }{ C \min_i {\norm{ \boldsymbol{x}^{(i)} }_2} }
    \end{align*}
    }
When $X$ has dimension $N>D$,
\footnotesize{
    \begin{align*}
    \kappa(\mathbf{H}) &\leq \frac{ C \norm{\mathbf{X}}_F}{\max_{\gamma \in \mathcal{O}}  
    \min_{n:H^{(n)} \in B_\gamma} \left(
    \norm{\mathbf{x}^{(n)}}_2^2 \  
        \min_{i} y_i^{(n)} \right) },\\
    \kappa(\mathbf{H}) &\geq    \frac{	\min_{n} \left( 
        \norm{\mathbf{x}^{(n)}}_2^2 \right) \  \max \left \{\left( \min_n  
        \max_i  
        y_i^{(n)}   \right)    \max_i \norm{\mathbf{x}^{(i)} }_2   ,   \left( 
        \max_{n,i}  y_i^{(n)}   \right)   \min_i  \norm{\mathbf{x}^{(i)} 
        }_2     \right \}   }{ C \max_{\alpha \in \mathcal{P} }   \norm { 
            \sum_{j: H^{(j)} \in B_\alpha} \mathbf{x}^{(j)}   }_2 } 
    \end{align*}
    }
\end{theorem}

\begin{appendices}
\section*{\appendixname: some technical lemmas}

The following lemma summarizes the translation invariance of $\boldsymbol\sigma$:
\begin{lemma}[Translational Invariance of Softmax]
  \label{lemma:sigma-translational-invariance}
  For every $\mathbf{u}\in\mathbb{R}^C$ and $c\in\mathbb{R}$
  \[ \boldsymbol\sigma(\mathbf{u} + c\ones) = \boldsymbol\sigma(\mathbf{u}). \]
  Conversely, if $\mathbf{u},\mathbf{v}\in\mathbb{R}^C$ and
  $\boldsymbol\sigma(\mathbf{v}) = \boldsymbol\sigma(\mathbf{u})$
  then there exists a $c\in\mathbb{R}$ such that $\mathbf{v} = \mathbf{u} + c\ones$.
  Similarly, if $\mathbf{U},\mathbf{V}\in L(\mathbb{R}^K,\mathbb{R}^C)$ then
  $\boldsymbol\sigma(\mathbf{V}) = \boldsymbol\sigma(\mathbf{U})$
  iff there exists a vector $\mathbf{c}\in\mathbb{R}^K$ such that
  \[ \mathbf{V} = \mathbf{U} + \ones\cdot\mathbf{c}^\intercal. \]
\end{lemma}
\begin{proof}
  Only the converse requires a proof.
  The equation $\boldsymbol\sigma(\mathbf{v}) = \boldsymbol\sigma(\mathbf{u})$ implies
  that for all $i$ we have $\exp(v_i) / b = \exp(u_i) / a$ where $a$ and $b$ are positive constants
  not depending on $i$. By taking logarithms of both sides we obtain $v_i = u_i + \log(b/a)$, or
  $\mathbf{v} = \mathbf{u} + \log(b/a)\ones$.
\end{proof}
This property of $\boldsymbol\sigma$ leads to the following
statement of translational invariance of $L(\mathbf{W})$:
\begin{lemma} For every $\mathbf{c}\in\mathbb{R}^D$
  \[ L(\mathbf{W} + \ones\cdot\mathbf{c}^\intercal) = L(\mathbf{W}). \]
  That is, we can add a constant to all entries in a column of
  $\mathbf{W}$ without changing the value of $L(\mathbf{W})$.
\end{lemma}

The following definition is known:
\begin{definition}[Stongly convex function]
  A differentiable function $f:U\to\mathbb{R}$, where $U\subseteq\mathbb{R}^n$ is and open set, is called \emph{strongly convex}
  iff there exists a number $m>0$ such that for all $\mathbf{x},\mathbf{y}\in\mathbb{R}^n$: 
  \[\langle\nabla f(x)-\nabla f(y),x-y\rangle\ge m\|x-y\|^2. \]
\end{definition}
It is clear (due to Mean Value Theorem) that a twice
continuously differentiable function is strongly convex iff
for every $\mathbf{x}\in U$ the bilinear form $D^2f(\mathbf{x})$
induces a positive definite quadratic form.

Strong convexity is too restrictive for our purposes: $L$ is not strongly convex.  We
adopted the following local notion:
\begin{definition}[Locally strongly convex function]
  A differentiable function $f:\mathbb{R}^n\to\mathbb{R}$ is called \emph{locally strongly convex}
  iff it is strongly convex in a neighborhood of every point $\mathbf{x}\in\mathbb{R}^n$.
\end{definition}

The following lemma is formulated in a notation that does not interfere
with any notations used in the paper.
\begin{lemma}[Criterion for Unique Global Minimum of a Convex Function]
  \label{lemma:criterion-of-unique-global-minimum}
  Let $f: X\to\mathbb{R}$ be a convex function, where $X\subseteq\mathbb{R}^n$ is a vector subspace.
  Then the following conditions are equivalent:
  \begin{enumerate}
  \item There exists a unique global minimum of $f$.
  \item For every $\mathbf{x}\in X$, $\mathbf{x}\neq 0$ we have $\lim_{\beta\to\infty} f(\beta\mathbf{x}) = \infty$.
  \item $\lim_{\mathbf{x}\to\infty} f(\mathbf{x}) = \infty$.
  \end{enumerate}
\end{lemma}
\begin{proof}
  Without loss of generality we may assume $X=\mathbb{R}^n$. Also,
  we may assume that $f$ is continuous, as every globally
  defined convex function on a finite dimensional space is continuous.

  We will prove $(1)\implies(3)\implies(2)\implies(1)$.

  $(1)\implies(3)$. We may assume that $\mathbf{0}$ is the unique
  global minimum and $f(\mathbf{0})=0$.  Thus $f > 0$ on the unit
  sphere. Let $m=\inf_{\mathbf{x}:\|x\|=1} f(\mathbf{x})$. Due to
  Bolzano-Weierstrass Theorem, $m>0$. For every $\mathbf{x}$ such
  that $\|\mathbf{x}\|\ge 1$ we have:
  \[\frac{\mathbf{x}}{\|\mathbf{x}\|} = (1-t)\mathbf{0}+t\mathbf{x},\quad 
  \text{where $t = \frac{1}{\|\mathbf{x}\|} \leq 1$}.\]
  Therefore, by definition of convexity,
  \[ m\le f\left(\frac{\mathbf{x}}{\|\mathbf{x}\|}\right)\leq (1-t) f(\mathbf{0}) + t f(\mathbf{x}) = t f(\mathbf{x}).\]
  Hence, $f(\mathbf{x}) \ge m\|\mathbf{x}\|$, which implies $\lim_{\mathbf{x}\to\infty}f(\mathbf{x})=\infty$.

  $(3)\implies(2)$. This is obvious.

  $(2)\implies(1)$. From the definition of convexity if follows
  that for every $\mathbf{x}\in\mathbb{R}^n$, $\mathbf{x}\neq 0$ the
  function $g(\beta) = f(\beta\mathbf{x})$ is a convex function
  $g:\mathbb{R}\to\mathbb{R}$. Hence $g'(\beta)$ is an increasing
  function and therefore $\lim_{\beta}g'(\beta)=M_1$ exists
  ($M_1=\infty$ is allowed). By replacing $\mathbf{x}$ with
  $-\mathbf{x}$ we conclude that $\lim_{\beta}g'(\beta)=M_2$
  exists. Obviously, $M_2\le M_1$. If either $M_1>0$ or
  $M_2 <0$ then $\lim_{\beta\to\pm\infty}$ is infinite, which
  is not possible by assumption. Hence $M_1=M_2=0$ and $g$ is
  thus constant. Hence $f$ is constant on every line passing
  through the origin. Hence $f(\mathbf{x})=f(\mathbf{0})$ for all
  $x\in\mathbb{R}^n$, i.e. $f$ is constant, contradicting the
  assumption.
\end{proof}

The following lemma allows calculations of limits of
$\boldsymbol\sigma$ along rays going to infinity:
\begin{lemma}
  \label{lemma:asymptotic-behavior-sigma}
  Let $\mathbf{u}\in\mathbb{R}^C$, $M=\max_i u_i$ and $J=\{i\,:\, u_i=M\}$. Then
  \[ \lim_{\beta\to\infty} \boldsymbol\sigma(\beta\mathbf{u}) = \frac{1}{|J|}\sum_{j\in J} \mathbf{e}_j \]
  where $\mathbf{e}_j$ denotes the $j$-th vector of the standard basis.
\end{lemma}
\begin{proof}
  To prove this claim, we notice that
  for $i\in J$
  \[ \sigma^{(i)}(\beta\mathbf{u}) = \frac{1}{|J| + \sum_{j\notin J} \exp(\beta(u_j-M))}. \]
  Hence, for $i\in J$, and $\lim_{\beta\to\infty}\exp(\beta(u_j-M)) = 0$ for $j\notin J$,
  \[ \lim_{\beta\to\infty} \sigma^{(i)}(\beta\mathbf{u}) = \frac{1}{|J|}. \]
  On the other hand, if $i\notin J$ then
  \[ \sigma^{(i)}(\beta\mathbf{u}) = \frac{ \exp(\beta(u_i-M))}{|J| + \sum_{j\notin J} \exp(\beta(u_j-M))}. \]
  Hence for $i\notin J$:
  \[ \lim_{\beta\to\infty} \sigma^{(i)}(\beta\mathbf{u}) = 0. \]                  
\end{proof}

\label{sec:technical-lemmas}
\end{appendices}



\bibliography{bibliography}

\end{document}